\documentclass[letterpaper, 10 pt, conference]{ieeeconf} 

\IEEEoverridecommandlockouts                          
\usepackage{amsthm}

\usepackage{amsmath} 
\usepackage{amssymb}  
\usepackage{tikz}
\usetikzlibrary{shapes,arrows}
\usepackage{graphicx} 
\usepackage{color} 
\usepackage{transparent}
\usepackage{cite}

\tikzstyle{block} = [draw, rectangle, minimum height=2em, minimum width=3em]
\tikzstyle{sum} = [draw, circle, node distance=1.5cm]
\tikzstyle{connection} = [draw, fill=black!100, circle, minimum size=3pt,inner sep=0, node distance=1cm]
\tikzstyle{input} = [coordinate]
\tikzstyle{output} = [coordinate]
\tikzstyle{agent} = [draw, thick, fill=blue!0, circle, minimum size=1em]

\pgfdeclarelayer{background}
\pgfdeclarelayer{foreground}
\pgfsetlayers{background,main,foreground}

\newcommand{\timeshift}{\mathrm{q}}

\theoremstyle{plain}
\newtheorem{theorem}{Theorem}

\newtheorem{definition}{Definition}

\newtheorem{remark}{Remark}

\title{\LARGE \bf
Distributed Iterative Learning Control for a Team of Quadrotors
}

\author{Andreas Hock and Angela P. Schoellig
\thanks{The authors are with the Dynamic Systems Lab (www.dynsyslab.org) at the University of Toronto Institute for Aerospace Studies (UTIAS), Canada. Email: andreas.hock@robotics.utias.utoronto.ca, schoellig@utias.utoronto.ca}
\thanks{This research was supported in part by NSERC grant RGPIN-2014-04634, the Connaught New Researcher Award and the Baden-W\"urttemberg-STIPENDIUM.}
}

\begin{document}

\maketitle
\thispagestyle{empty}
\pagestyle{empty}

%%%%%%%%%%%%%%%%%%%%%%%%%%%%%%%%%%%%%%%%%%%%%%%%%%%%%%%%%%%%%%%%%%%%%%%%%%%%%%%%
\begin{abstract}

The goal of this work is to enable a team of quadrotors to learn how to accurately track a desired trajectory while holding a given formation. We solve this problem in a distributed manner, where each vehicle has only access to the information of its neighbors. The desired trajectory is only available to one (or few) vehicles. We present a distributed iterative learning control (ILC) approach where each vehicle learns from the experience of its own and its neighbors' previous task repetitions, and adapts its feedforward input to improve performance. Existing algorithms are extended in theory to make them more applicable to real-world experiments. In particular, we prove stability for any causal learning function with gains chosen according to a simple scalar condition. Previous proofs were restricted to a specific learning function that only depends on the tracking error derivative (D-type ILC). Our extension provides more degrees of freedom in the ILC design and, as a result, better performance can be achieved. We also show that stability is not affected by a linear dynamic coupling between neighbors. This allows us to use an additional consensus feedback controller to compensate for non-repetitive disturbances. Experiments with two quadrotors attest the effectiveness of the proposed distributed multi-agent ILC approach. This is the first work to show distributed ILC in experiment.

\end{abstract}

%%%%%%%%%%%%%%%%%%%%%%%%%%%%%%%%%%%%%%%%%%%%%%%%%%%%%%%%%%%%%%%%%%%%%%%%%%%%%%%%
\section{INTRODUCTION}
Multi-agent systems (MAS) and machine learning are two exciting trends in robotics. In the past decades, theoretic contributions on MAS have come from fields such as biology, computer science, and control theory. One problem of particular interest is \textit{consensus}, which is concerned with all agents agreeing on some quantity of interest by only communicating with their neighbors. Many other problems can be transformed into a consensus problem; examples include flocking, rendezvous, or formation control \cite{WeiRen2005}. Consensus can be achieved without a central control unit through the design of appropriate distributed algorithms. Machine learning, on the other hand, aims to enhance the capabilities of autonomous systems by enabling them to adapt to unknown situations, autonomously correct for modeling errors, and improve their performance without human instructions. As the number of autonomous systems increases in all areas, including industrial and service robots, commercial drones, and self-driving cars, the question that arises is how their cooperation can be improved and, therefore, how MAS and machine learning can be combined.

In this paper, we focus on Iterative Learning Control (ILC) approaches, where a system learns to track a desired trajectory by repetitively executing the same task. Based on the tracking error from previous trials, the feedforward input is adapted to gradually improve performance. As such, ILC is able to compensate for repetitive disturbances. Initially developed in 1984 \cite{Arimoto1984}, it has since been studied widely in theory and experiments, cf. \cite{Bristow2006b}. In \cite{Schoellig2012a}, ILC was used to improve the trajectory tracking of a single quadrotor. 

Distributed ILC achieves formation control of MAS where only a subset of agents has direct access to the reference trajectory and only neighboring agents can communicate, see Figure \ref{fig:QuadGraph}. The goal is to follow the reference with all agents holding a predefined formation. To achieve this, every agent updates its input trajectory based on the information about its own and its neighbors' performance during the last trial. The idea of distributed ILC was introduced in 2009 \cite{Ahn2009a}, where stability was proven for a D-type input update rule; that is, the input for the next iteration is computed based on the previous input and the derivative of the tracking error. Furthermore, this first paper assumes communication graphs in the form of directed spanning trees. The heterogeneous agent dynamics are described in continuous time and are assumed to be nonlinear with relative degree one. In \cite{Yang2012c}, the proof was extended to arbitrary directed graphs and the stability criterion was simplified for the case of homogeneous agents. Even time-varying dynamics can be handled with the approach in \cite{Yang2012c}. To extend the algorithm to systems of higher relative degree, higher derivatives were used in the input update rule, see \cite{Meng2012d}. This approach is, however, restricted to linear agent dynamics, but holds for weighted and directed communication graphs.
\begin{figure} 
	\centering 
	\def\svgwidth{200pt} 
	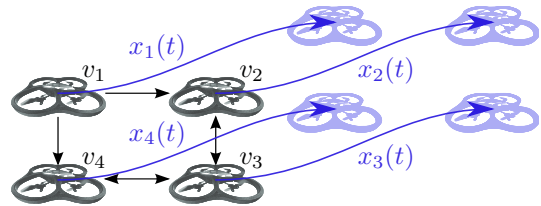
	\caption{A team of quadrotors learns to accurately track a desired trajectory while holding a given formation. The black arrows between the vehicles~$v_i$ represent the communication paths. The blue arrows depict the vehicle trajectories $x_i(t)$.}
	\label{fig:QuadGraph}
\end{figure}
All mentioned papers (\cite{Ahn2009a, Yang2012c, Meng2012d}) are using the same D-type input update rule. The shortcoming of this control strategy is a poor convergence behavior; in particular, a position offset cannot be compensated for as D-type ILC only adapts the input based on errors in the velocity.

In this paper, we prove that for linear agents any causal learning function,
with gains chosen according to a simple scalar condition, is stable. This allows
more options in the choice of the input-update rule (beyond the typical D-type update rule) and thus faster convergence can be achieved. Moreover, a constant error in the position can now be compensated for by incorporating position information into the learning function.
 
Several other multi-agent ILC approaches and extensions are found in the literature including approaches for switching graph structures \cite{Liu2012a}, and adaptive \cite{Li2014b}, robust \cite{Meng2015a}, or \mbox{optimal \cite{Yang2014a}} update laws. Nevertheless, all of these approaches have only been implemented in simulation and, to the authors' best knowledge, multi-agent ILC has not been tested on a real system before.

This paper demonstrates the proposed generalized multi-agent ILC algorithm in experiment on quadrotor vehicles. To increase robustness against non-repetitive disturbances occurring in real-world experiments, we include a distributed feedback controller, which improves the performance during single iterations. Furthermore, we show that neither the distributed feedback controller nor any other linear dynamic coupling between neighboring agents affects stability of the ILC algorithm.

The paper is structured as follows: First, we introduce some definitions and basic terminology from graph theory. In Section~III, formation tracking is formulated as a consensus problem. Then, we derive the ILC stability conditions for an arbitrary linear causal learning function in Section~IV. Based on these results, we show that a dynamic coupling between neighboring agents does not influence stability. In Section~VI, we apply the theory to a team of two quadrotors and show corresponding experimental results. Finally, conclusions are provided in Section~VII.

%%%%%%%%%%%%%%%%%%%%%%%%%%%%%%%%%%%%%%%%%%%%%%%%%%%%%%%%%%%%%%%%%%%%%%%%%%%%%%%%
\section{PRELIMINARIES ON GRAPH THEORY}
To describe the information exchange between agents, graph theory is commonly used. The vehicles are the nodes of the graph and the edges represent the information flow between vehicles, see Figure \ref{fig:QuadGraph}. In the following, we provide some useful definitions and notation. 

Let $\mathcal{G}=(\mathcal{V},\mathcal{E},\mathcal{A})$ denote a directed graph with a set of vertices $\mathcal{V}(\mathcal{G})=\{v_i:i\in \{1,2,...,N\}\}$ and the edge set $\mathcal{E}(\mathcal{G})\subseteq \{(v_i,v_j): v_i,v_j \in \mathcal{V}(\mathcal{G})\}$. The edge $(v_i,v_j)$ means that agent $v_i$ receives information from $v_j$, in which case $v_i$ is called the child and $v_j$ the parent vertex. We define $\mathcal{A}= (a_{ij})\in \mathbb{R}^{N\times N}$ as the adjacency matrix of $\mathcal {G}$ with elements representing the information exchange between any two agents; that is, $a_{ij}=1$ if $(v_i,v_j)\in \mathcal{E(G)}$ and $a_{ij}=0$ otherwise.
The in-degree of node $v_i$ is defined as $d_i^\mathrm{in} = \sum_{j=1}^N a_{ij}$; that is, it describes the number of edges entering a node. The (in-degree) Laplacian matrix is defined as $\mathcal{L}_\mathcal{G}=\mathcal{D}-\mathcal{A}$, where $\mathcal{D} = diag(d_1^\mathrm{in},d_2^\mathrm{in},...,d_N^\mathrm{in})$. If there exists a special vertex, the root, with no parents ($d^\mathrm{in}=0$) which all other nodes can be connected to through directed paths and every other vertex has exactly one parent, then the graph is called a spanning tree, see Figure~\ref{fig:SpanningTree}. Note that some authors also use the term 'rooted out-branching' to distinguish between directed and undirected graphs. For further information about graph theory, see \cite{Mesbahi2010}. 

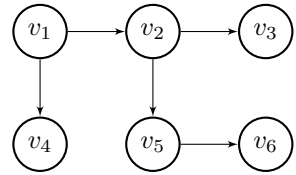
\begin{figure} 
	\vspace{2mm}
	\centering 
	\begin{tikzpicture}[auto, node distance=1.5cm,>=latex']
	\node [agent, name=v1] {$v_1$};
	\node [agent,right of=v1, name=v2] {$v_2$};
	\node [agent,right of=v2, name=v3] {$v_3$};
	\node [agent,below of=v1, name=v4] {$v_4$};
	\node [agent,below of=v2, name=v5] {$v_5$};
	\node [agent,right of=v5, name=v6] {$v_6$};
	\draw [->] (v1) -- (v2);
	\draw [->] (v1) -- (v4);
	\draw [->] (v2) -- (v5);
	\draw [->] (v2) -- (v3);
	\draw [->] (v5) -- (v6);
	\end{tikzpicture}
	\caption{Spanning tree structure of a directed graph. Node~$v_1$ has no incoming edges and is called the root node. All other vertices have exactly one parent (one incoming edge) and can be connected to the root through directed paths.}
	\label{fig:SpanningTree}
\end{figure}

%%%%%%%%%%%%%%%%%%%%%%%%%%%%%%%%%%%%%%%%%%%%%%%%%%%%%%%%%%%%%%%%%%%%%%%%%%%%%%%%
\section{PROBLEM FORMULATION}
 We consider a group of $N$ homogeneous agents with fixed interaction topology defined by $\mathcal{G}=(\mathcal{V},\mathcal{E},\mathcal{A})$.

\subsection{Agent Dynamics}
The linear discrete-time single-input, single-output (SISO) dynamics of the \textit{i}th agent at the $k$th iteration of the task, $k\in \{1,2,\dots\}$, and for sampling times $t\in \{0,1,\dots,T\}$ are described by
 \begin{align}
 x_{i,k}(t+1) &= Ax_{i,k}(t)+Bu_{i,k}(t)\nonumber \\
 y_{i,k}(t) &=Cx_{i,k}(t) \label{statespace},
 \end{align}
 where $x_{i,k}(t)\in \mathbb{R}^n$ is the state vector, $u_{i,k}(t)\in \mathbb{R}$ the input, and $y_{i,k}(t)\in \mathbb{R}$ the output. Accordingly, ${A \in \mathbb{R}^{n\times n}}$, ${B \in \mathbb{R}^{n}}, C \in \mathbb{R}^{1\times n}$.\\
 Using the time-shift operator $\timeshift^{-1}$, defined by ${\timeshift^{-1} x(t)=x(t-1)}$, the state-space system \eqref{statespace} can be represented as 
 \begin{align}
  y_{i,k}(t) &=P(\timeshift)u_{i,k}(t) + d_i(t) \label{LTI-plant},
 \end{align}
 where $d_i(t)=C A^t x_i(0)$ is the free response to the initial condition, which is assumed to be constant over iterations. The input-output mapping $P(\timeshift)$ is given by an infinite power series,
 \begin{align}
  P(\timeshift)=p_1\timeshift^{-1}+p_2\timeshift^{-2}+p_3\timeshift^{-3}+\dots\,, \label{P(q)}
 \end{align}
 with Markov parameters $p_m = CA^{m-1}B$. The relative degree $r$ of system \eqref{statespace} is defined by the first non-zero coefficient; that is, $p_m=0$ for $m<r$, $p_m\neq0$ for $m=r$.
 \subsection{Reference Trajectory} 
 A reference trajectory $y_{des}(t)$ is given, which a certain subset of agents has access to. The goal is to track this reference signal with all agents simultaneously. The result can be directly applied to formation control by simply defining fixed or time-varying relative distances between the agents.
 We model the reference as an additional node $v_0$, the virtual leader. This results in the extended graph $\mathcal{G}^*$. Let $b_i\in \{0,1\}$ be defined for every agent analogously to the entries of the adjacency matrix: $b_i=1$ if agent \textit{i} has access to the reference trajectory, and $b_i =0$ otherwise. The corresponding Laplacian $\mathcal{L_{G^*}}$, in the following simply denoted by $\mathcal{L}^*$, is
 \begin{equation}
 \mathcal{L^*}=
 \begin{bmatrix}
 0& 0_{1\times N}\\
 -\mathrm{b} & \mathcal{L_G}+\mathcal{B}
 \end{bmatrix}
 \label{eq:L*_definition}
 \end{equation}
with $\mathrm{b}=(b_1,b_2,\dots,b_N)^T$, $\mathcal{B} = diag(\mathrm{b})$ and $0_{1\times N}$ denoting a matrix of size $(1\times N)$ with all entries being zero.
\subsection{Error Signals}
The only information exchanged between agents is the relative position, which can be easily obtained in practical implementations. For example, an agent's own camera system could detect the neighbors and measure the distances; no communication between agents would be necessary. 
Let us define an error signal $e_{i,k}(t)$ for each agent as the sum of relative distances to all its neighbors and to the virtual leader, if accessible:
 \begin{equation}
e_{i,k}(t)=\sum_{j=1}^N a_{ij}(y_{j,k}(t)-y_{i,k}(t))+b_i(y_{des}(t)-y_{i,k}(t)).
 \label{e(t)}
 \end{equation}
 \begin{remark}
 	The error function \eqref{e(t)} is for the consensus case, in which all agents aim to follow the same trajectory. If the goal is a desired constant or time-varying formation, additional iteration-invariant terms, the desired distances~$\Delta_{ij}(t)$, have to be added. Thus,
 	 \begin{align}
 	 e_{i,k}(t)= &\sum_{j=1}^N a_{ij}(y_{j,k}(t)-y_{i,k}(t)+\Delta_{ij}(t))\nonumber\\
 	 &+b_i(y_{des}(t)-y_{i,k}(t)+\Delta_{i}(t)).
 	 \label{e(t)_with distances}
 	 \end{align} 
 \end{remark}
%%%%%%%%%%%%%%%%%%%%%%%%%%%%%%%%%%%%%%%%%%%%%%%%%%%%%%%%%%%%%%%%%%%%%%%%%%%%%%%%
\section{DISTRIBUTED LEARNING}
We use the distributed input update rule
\begin{align}
u_{i,k+1}(t)=u_{i,k}(t)+L(\timeshift)e_{i,k}(t+r)
\label{inputupdate}
\end{align}
with $i\in \{1,\dots,N\} \text{ and } t\in \{1,\dots,T-r\}$ and causal learning function
\begin{equation}
L(\timeshift)=l_0+l_1\timeshift^{-1}+l_2 \timeshift^{-2}+\dots,
\label{L(q)}
\end{equation}
 where $r$ is the relative degree. The block diagram is shown in Figure \ref{fig:ILC}. In the following, the relative degree is assumed to be one without loss of generality. A higher relative degree can be compensated for by a corresponding time shift of the error vector. The acausality in the update rule is not a problem as the input update is computed after the previous execution was completed. Algorithm \eqref{inputupdate} is a simple, first-order ILC algorithm commonly used in literature, see \cite{Bristow2006b}. The order of an ILC algorithm is defined as the number of previous iterations taken into account.

D-type ILC algorithms as in \cite{Ahn2009a,Yang2012c,Meng2012d} represent special cases of \eqref{inputupdate}, where $L(\timeshift)$ has a predefined form.
\begin{figure}
	\centering
	\vspace{2mm}
	\begin{tikzpicture}[auto, node distance=1.5cm,>=latex']
		\node [input, name=input] {};
		\node [block, right of=input] (memory1) {Memory};
		\node [block, right of=memory1] (learningfcn) {L(q)};
		\node [sum, right of=learningfcn] (sumILC) {};
		\node [block, above of=sumILC, node distance=1.cm] (memory2) {Memory};
		\node [connection, right of=sumILC, node distance=1.cm] (connec) {};
		\node [block, right of=connec, node distance=1.cm] (plant) {P(q)};
		\node [output, right of=plant] (output) {};
	
		\draw [draw,->] (input) -- node [name=e] {$e_i$} (memory1);
		\draw [-] (sumILC) -- node {$u_i$} (connec);
		\draw [->] (connec) -- (plant);
		\draw [->] (connec) |- (memory2);
		\draw [->] (plant) -- node [name=y] {$y_i$}(output);
		\draw [->] (memory1) -- (learningfcn);
		\draw [->] (learningfcn) --  (sumILC);
		\draw [->] (memory2) -- (sumILC); %node [below,pos=0.9]{$\tiny{+}$}
		\end{tikzpicture}
	\caption{Basic ILC structure for vehicle $v_i$ with the plant $P(\timeshift)$ and the learning function $L(\timeshift)$. The input is computed based on the previous input and previous consensus tracking error $e_i$, which are saved in memory units.}
	\label{fig:ILC}
\end{figure}
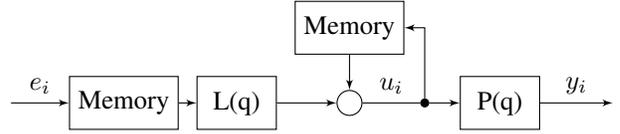

For analyzing stability, we use the lifted-system representation according to \cite{Bristow2006b}, where all samples of a signal are stacked in a large vector. The system dynamics \eqref{LTI-plant} and the input update~\eqref{inputupdate} are now represented by
\begin{align}
\underbrace{
	\begin{bmatrix}
	y_{i,k}(1)\\y_{i,k}(2)\\ \vdots \\ y_{i,k}(T)
	\end{bmatrix}}_{\mathrm{y}_{i,k}}
=&\mathrm{P}
\underbrace{
	\begin{bmatrix}
	u_{i,k}(0)\\u_{i,k}(1)\\ \vdots \\ u_{i,k}(T-1)
	\end{bmatrix}
}_{\mathrm{u}_{i,k}}
+\underbrace{\begin{bmatrix}
	d_{i}(1)\\ d_{i}(2)\\ \dots \\d_{i}(T)
	\end{bmatrix}}_{\mathrm{d}_{i}},\label{SD_inLFS}\\
\underbrace{
	\begin{bmatrix}
	u_{i,k+1}(0)\\u_{i,k+1}(1)\\ \vdots \\ u_{i,k+1}(T-1)
	\end{bmatrix}
}_{\mathrm{u}_{i,k+1}}
=&
\underbrace{
	\begin{bmatrix}
	u_{i,k}(0)\\u_{i,k}(1)\\ \vdots \\ u_{i,k}(T-1)
	\end{bmatrix}
}_{\mathrm{u}_{i,k}}
+
\mathrm{L}\underbrace{\begin{bmatrix}
	e_{i,k}(1)\\e_{i,k}(2)\\ \vdots \\ e_{i,k}(T)
	\end{bmatrix}}_{\mathrm{e}_{i,k}}
\label{IUR_inLFS}
\end{align}
with lower-triangular Toeplitz matrices
\begingroup
\begin{equation*}
\mathrm{P}=\begin{bmatrix}
		p_1 & 0 & \dots & 0\\
		p_2 & p_1 & \dots & 0\\
		\vdots & \vdots & \ddots & \vdots \\
		p_T & p_{T-1} & \dots & p_1
	\end{bmatrix}\hspace{-1mm},
\mathrm{L}=\begin{bmatrix}
		l_0 & 0 & \dots & 0\\
		l_1 & l_0 & \dots & 0\\
		\vdots & \vdots & \ddots & \vdots \\
		l_{T-1} & l_{T-2} & \dots & l_0
	\end{bmatrix}\hspace{-1mm}.
\end{equation*}
\endgroup
We combine all single-agent dynamics into one equation
\begin{align}
\underbrace{
	\begin{bmatrix}
	\mathrm{y}_{1,k}\\\mathrm{y}_{2,k}\\ \vdots \\ \mathrm{y}_{N,k}
	\end{bmatrix}}_{\mathrm{Y}_k}
=&
\underbrace{
	\begin{bmatrix}
	\mathrm{P} & 0 & \dots & 0\\
	0 & \mathrm{P} & \dots & 0\\
	\vdots & \vdots & \ddots & \vdots \\
	0 & \dots & \dots & \mathrm{P}
	\end{bmatrix}
}_{\mathrm{I}_N\otimes\mathrm{P}}
\underbrace{
	\begin{bmatrix}
	\mathrm{u}_{1,k}\\\mathrm{u}_{2,k}\\ \vdots \\ \mathrm{u}_{N,k}
	\end{bmatrix}
}_{\mathrm{U}_{k}}
+\underbrace{\begin{bmatrix}
	\mathrm{d}_{1}\\ \mathrm{d}_{2}\\ \vdots \\\mathrm{d}_{N}
	\end{bmatrix}}_{\mathrm{D}}
\label{eq:dynamics_LSF_MAS},
\end{align}
where $\otimes$ denotes the Kronecker product and $\mathrm{I}_N$ the \mbox{$(N\times N)$} identity matrix.
Analogously, using \eqref{e(t)} and the graph-theoretic definitions from Section II, the multi-agent version \mbox{of \eqref{IUR_inLFS}} is
\begin{align}
\mathrm{U}_{k+1} = &\mathrm{U}_k -(\mathrm{I}_N\otimes\mathrm{L})\,\cdot\nonumber\\ &\Big(\big((\mathcal{L_G+B})\otimes\mathrm{I}_T\big)\mathrm{Y}_k -(\mathrm{b}\otimes \mathrm{I}_T)\mathrm{y}_{des}\Big)
\label{eq:ILC_LSF_MAS}
\end{align}
with $\mathrm{y}_{des}=(y_{des}(1),y_{des}(2),\dots,y_{des}(T))^T$. For the sake of simplicity, we collect all iteration-invariant terms, that is, terms only depending on the initial conditions or the reference trajectory, in $\delta$. Then, by plugging~\eqref{eq:dynamics_LSF_MAS} into~\eqref{eq:ILC_LSF_MAS} and using the property of the Kronecker product that ${(A\otimes B)(C\otimes D)=AC\otimes BD}$, it follows
\begin{align}
\mathrm{U}_{k+1} 
&=\mathrm{U}_k - (\mathrm{I}_N\otimes\mathrm{L})\big((\mathcal{L_G+B})\otimes\mathrm{I}_T\big) (\mathrm{I}_N\otimes\mathrm{P})\mathrm{U}_k+\delta\nonumber\\
&= \mathrm{U}_k - \big(\mathrm{I}_N(\mathcal{L_G+B})\mathrm{I}_N\otimes\mathrm{L}\mathrm{I}_T\mathrm{P}\big)\mathrm{U}_k +\delta \nonumber\\
&= \big(\mathrm{I}_{NT}-(\mathcal{L_G+B})\otimes\mathrm{L}\mathrm{P}\big)\mathrm{U}_k + \delta.
\label{eq:U_ILCdynamics}
\end{align}
Based on stability theory for discrete systems, conditions for the stability of single-agent ILC were developed in \cite{Norrlof2002} and slightly modified in \cite{Bristow2006b}. The following definitions are taken from the latter and adapted to MAS.
\begin{definition}
	An ILC system is called \textit{asymptotically stable} if there exists $\overline{\mathrm{U}}\in \mathbb{R}^{NT}$ such that $\forall k=\{0,1,\dots\}$
	\begin{equation*}
	|\mathrm{U}_k|\leq \overline{\mathrm{U}}   
	\quad \text{and} \quad \lim\limits_{k \to \infty}\mathrm{U}_k \quad \text{exists.}
	\end{equation*}
\end{definition}
Using this definition, equation \eqref{eq:U_ILCdynamics}, and the notion of the spectral radius $\rho$ as the maximum absolute eigenvalue, we can state the following theorem.
\begin{theorem}
	The multi-agent ILC system \eqref{eq:dynamics_LSF_MAS}-\eqref{eq:U_ILCdynamics} is AS if and only if
	\begin{equation}
	\rho\big(\mathrm{I}_{NT}-(\mathcal{L_G+B})\otimes\mathrm{L}\mathrm{P}\big)<1
	\label{eq:stabilitycondition_ILC}
	\end{equation}	
	or, equivalently,
	\begin{equation}
	\max_i|1-\lambda_i l_0 p_1|<1
	\label{eq:stabilitycondition_ILC_scalar},
	\end{equation}
	where $\lambda_i$ are the eigenvalues of $(\mathcal{L_G+B})$. 
		\label{theorem_stabilitycondition_ILC}
\end{theorem}
Note that we assumed a relative degree $r=1$, and thus $p_1\neq 0$. The eigenvalues of the graph Laplacian can be computed easily. As stability depends only on $l_0$, the remaining parameters of the learning function $L(\timeshift)$ can be tuned to achieve good convergence behavior. 
\begin{proof}
	The first statement \eqref{eq:stabilitycondition_ILC} follows directly from \cite{Norrlof2002} applied to \eqref{eq:U_ILCdynamics}. The latter equation \eqref{eq:stabilitycondition_ILC_scalar} is obtained by applying a similarity transformation to (14) similarly to how it is done for undirected graphs in \cite{Yang2012c}.
	
	As all entries in $(\mathcal{L_G+B})$ are real numbers, this matrix can be transformed into Jordan normal form; that is, it exists a regular matrix $ S\in \mathbb{R}^{N\times N}$ and a Jordan matrix $J \in \mathbb{R}^{N\times N}$ with eigenvalues on the diagonal and possibly ones on the subdiagonal, such that
	\begin{equation}
	S^{-1}(\mathcal{L_G+B})S = J.
	\label{eq:similarity_transformation}
	\end{equation}
	Usually $J$ is defined with ones on the superdiagonal but here we use this less common definition to get lower-triangular matrices.
	As eigenvalues and, thus, the spectral radius remain the same under similarity transformations, it follows
	\begin{align}
	&\rho\big(\mathrm{I}_{NT}-(\mathcal{L_G+B})\otimes\mathrm{L}\mathrm{P}\big)\nonumber\\ 
	=&\rho\Big((S\otimes \mathrm{I}_T)^{-1} \big(\mathrm{I}_{NT}-(\mathcal{L_G+B})\otimes\mathrm{L}\mathrm{P}\big)(S\otimes\mathrm{I}_T)\Big)\nonumber\\
	=& \rho(\mathrm{I}_{NT}-J\otimes\mathrm{L}\mathrm{P}),
	\end{align}
	where L and P are lower triangular matrices, as is $J$. As a result, the eigenvalues are the diagonal entries. Multiplication of triangular matrices does not affect this property, thus \eqref{eq:stabilitycondition_ILC} is equivalent to the scalar condition \eqref{eq:stabilitycondition_ILC_scalar}.
\end{proof}

\begin{theorem}
	For asymptotic stability of the ILC, it is necessary that the graph $\mathcal{G}^*$ contains a spanning tree with the virtual leader as root.
	\label{theorem_spanningTree}
\end{theorem}
\begin{proof}
It is easy to see that \eqref{eq:stabilitycondition_ILC_scalar} only holds if $\lambda_i \neq 0$. $\mathcal{L}^*$ has exactly one eigenvalue at zero if and only if the graph~$\mathcal{G}^*$ contains a spanning tree \cite{Mesbahi2010}. Thus, the virtual leader must be the root node as it has an in-degree of 0. Therefore, $(\mathcal{L_G+B})$ is full rank, see \eqref{eq:L*_definition}, and equivalently, $\lambda_i \neq 0.$
\end{proof}
%Furthermore, using the Gershgorin circle theorem \cite{Ren2008a}, the spectrum of $(\mathcal{L_G+B})$ can be bounded as
%\begin{equation} 
%0<\lambda_i\leq 2(N-1)+1.\label{eq:Gershgorin_bound}
%\end{equation}
% This can be derived from the special structure of the Laplacian, where $N-1$ is the maximum number of incoming edges, thus an upper bound on the row sum. The diagonal entries can also be bounded by this number increased by one for the possible influence of the virtual leader. From the proof of Theorem~\ref{theorem_spanningTree}, it follows that $\lambda_i>0$.
%\begin{corollary}
% Under the condition stated in Theorem~\ref{theorem_spanningTree}, the multi-agent ILC algorithm \eqref{inputupdate} is asymptotically stable for all possible communication topologies if the learning function \eqref{L(q)} is chosen such that
% \begin{equation*}
% \mathrm{sgn}(l_0)= \mathrm{sgn}(p_1)\quad \text{and} \quad |l_0|<\frac{2}{(2N-1) |p_1|}.
% \end{equation*}  
%\end{corollary}
\begin{remark}
	The extension of these results to quadratic multiple-input, multiple-output (MIMO) agent dynamics is straightforward. Condition \eqref{eq:stabilitycondition_ILC_scalar} stays the same, but is no longer scalar as $l_0$ and $p_0$ are quadratic matrices.
\end{remark}
%\begin{remark}
%	For linear time-varying (LTV) agent dynamics, \eqref{eq:stabilitycondition_ILC} still holds, but L and P are no longer Toeplitz.
%\end{remark}
%%%%%%%%%%%%%%%%%%%%%%%%%%%%%%%%%%%%%%%%%%%%%%%%%%%%%%%%%%%%%%%%%%%%%%%%%%%%%%%%
\section{COMBINATION WITH FEEDBACK}
Assume there is an additional feedback term in the time domain that can be described as a function of the relative distances between neighbors,
\begin{align}
u^{\mathrm{FB}}_{i,k}(t)=C(\timeshift)e_{i,k}(t)
\label{eq:u_FB}
\end{align}
with $e_{i,k}(t)$ as in \eqref{e(t)} and a causal feedback function
\begin{align}
C(\timeshift) = c_0+c_1\timeshift^{-1}+c_2\timeshift^{-2}+\dots .
\end{align}
This could be a feedback controller or any other dynamic coupling, which may not be known.
\begin{theorem}
	Given an arbitrary feedback component in the form of \eqref{eq:u_FB}, the stability of the distributed ILC algorithm \eqref{inputupdate} is not affected by the feedback if applied in parallel structure, see \cite{Bristow2006b} and Figure \ref{fig:ILCwFB}.
\end{theorem} 
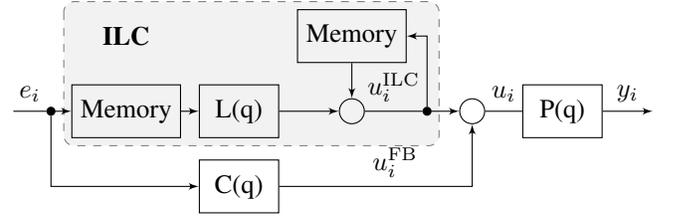
\begin{figure}
	\vspace{2mm}
	\centering
	\begin{tikzpicture}[auto, node distance=1.5cm,>=latex']
	\node [input, name=input] {};
	\node [connection, right of=input, node distance=0.5cm] (connec) {};
	\node [block, right of=connec, node distance=1cm] (memory1) {Memory};
	\node [block, right of=memory1] (learningfcn) {L(q)};
	\node [block, below of=learningfcn, node distance=1cm] (controller) {C(q)};
	\node [sum, right of= learningfcn] (sumILC) {};
	\node [connection, right of=sumILC, node distance=1cm] (connec2) {};
	\node [sum, right of=connec2, node distance=0.6cm] (sum) {};
	\node [above of=memory1, node distance=1cm] (ILC) {\textbf{ILC}};
	\node [block, above of=sumILC, node distance=1cm] (memory2) {Memory};	
	\node [block, right of=sum, node distance=1.2cm] (plant) {P(q)};
	\node [output, right of=plant, node distance=1.2cm] (output) {};
	\draw [draw,-] (input) -- node [name=e] {$e_i$} (connec);
	\draw [->] (connec) |- (controller);
	\draw [->] (controller) -| node [above, pos=.3] {$u_i^{\mathrm{FB}}$} (sum);
	\draw [-] (sumILC) -- node {$u_i^{\mathrm{ILC}}$} (connec2);
	\draw [->] (connec2) -- (sum);
	\draw [-] (sum) -- node {$u_i$} (plant);
	\draw [->] (connec2) |- (memory2);
	\draw [->] (plant) -- node [name=y] {$y_i$}(output);
	\draw [->] (connec) -- (memory1);
	\draw [->] (memory1) -- (learningfcn);
	\draw [->] (learningfcn) -- (sumILC);
	\draw [->] (memory2) -- (sumILC);
	\begin{pgfonlayer}{background}
	% Compute a few helper coordinates
	\path (memory1.west |- memory2.north)+(-0.1,0.1) node (a) {};
	\path (memory1.south -| connec2.east)+(+0.1,-0.1) node (b) {};
	\path[fill=black!5,rounded corners, draw=black!50, dashed]
	(a) rectangle (b);
	\end{pgfonlayer}
	\end{tikzpicture}
	\caption{Time-domain feedback control combined with iteration-domain ILC for vehicle $v_i$. The feedback term $C(q)$ computes updates in every time step, while the ILC part computes updates after each iteration. }
	\label{fig:ILCwFB}
\end{figure}
\begin{proof}
The new input signal is $u_{i,k}=u_{i,k}^{\mathrm{ILC}}+u_{i,k}^{\mathrm{FB}}$, where $u_{i,k}^{\mathrm{ILC}}$ represents the ILC input described in Section~IV. Using again the lifted-system representation and inserting into \eqref{SD_inLFS} yields
\begin{align}
\mathrm{y}_{i,k}
=&
\mathrm{P}
\left( \mathrm{u}^{\mathrm{ILC}}_{i,k} + \mathrm{C}\mathrm{e}_{i,k}+ C_0 e_i(0)\right)
+\mathrm{d}_{i,k},
\label{eq:SD_withFB_inLFS}
\end{align} 
with
\begin{align*}
\mathrm{C}=\begin{bmatrix}
0 & 0 & \dots & 0&0\\
c_0 & 0 & \dots & 0&0\\
c_1 & c_0& \dots & 0&0\\
\vdots & \vdots & \ddots & \vdots& \vdots \\
c_{T-2} & c_{T-3} & \dots & c_0 &0
%0 & \dots & 0&0\\
%c_0 & \dots & 0&0\\
%\vdots & \ddots & \vdots& \vdots \\
%c_{T-2} & \dots & c_0 &0
\end{bmatrix}, 
\mathrm{C}_0=\begin{bmatrix}
c_0\\
c_1\\
c_2\\
\vdots\\
c_{T-1}
\end{bmatrix}.
\end{align*}
The different sample times in the definitions of $\mathrm{e}_{i,k}$ and $\mathrm{u}_{i,k}$ cause the additional initial condition term $C_0 e_i(0)$ in~\eqref{eq:SD_withFB_inLFS} and the subdiagonal shift in $\mathrm{C}$. Analogously to before, the single-agent signals can be collected into the multi-agent form
\begin{align}
\mathrm{Y}_{k}
=&
(\mathrm{I}_N\otimes\mathrm{P})
\Big( \mathrm{U}^{\mathrm{ILC}}_{k} - (\mathrm{I}_N\otimes\mathrm{C})\big(\mathcal{(L_G+B)}\otimes \mathrm{I}_T\big)\Big)\mathrm{Y}_k+\delta
\nonumber\\
=&(\mathrm{I}_N\otimes\mathrm{P}) \mathrm{U}^{\mathrm{ILC}}_{k}
-\big(\mathcal{(L_G+B)}\otimes \mathrm{PC}\big)\mathrm{Y}_k+\delta
\nonumber\\
=&\big(\mathrm{I}_N+\mathcal{(L_G+B)}\otimes \mathrm{PC}\big)^{-1}
(\mathrm{I}_N\otimes\mathrm{P}) \mathrm{U}^{\mathrm{ILC}}_{k}+\delta,
\label{eq:Y_withFB_MAS_LSF}
\end{align}
where all iteration-invariant terms are gathered in $\delta$. Invertibility is guaranteed as we will see later. Inserting~\eqref{eq:Y_withFB_MAS_LSF} into~\eqref{eq:ILC_LSF_MAS} leads to 
\begin{align}
\mathrm{U}^{\mathrm{ILC}}_{k+1} =& \mathrm{U}^{\mathrm{ILC}}_k - (\mathrm{I}_N\otimes\mathrm{L})\big((\mathcal{L_G+B})\otimes\mathrm{I}_T\big)
\nonumber \\
&\big(\mathrm{I}_N+\mathcal{(L_G+B)}\otimes \mathrm{PC}\big)^{-1}
(\mathrm{I}_N\otimes\mathrm{P}) \mathrm{U}^{\mathrm{ILC}}_{k}+\delta \nonumber\\
=& \bigg(\mathrm{I}_{NT} - \Big(\big((\mathcal{L_G+B})\otimes\mathrm{L}\big)\nonumber\\
&\big(\mathrm{I}_N+\mathcal{(L_G+B)}\otimes \mathrm{PC}\big)^{-1}
(\mathrm{I}_N\otimes\mathrm{P})\Big)\bigg) \mathrm{U}^{\mathrm{ILC}}_{k}+\delta \nonumber\\
=& M\mathrm{U}^{\mathrm{ILC}}_{k}+\delta.
\label{eq:U_ILC_dynamics_withFB}
\end{align}
As stability is determined by the spectral radius $\rho(M)$, we must investigate the eigenvalues of this matrix. The similarity transformation \eqref{eq:similarity_transformation} can be applied to get
\begin{align}
\rho(M)
=\rho(\mathrm{I}_{NT} - (J\otimes\mathrm{L})(\mathrm{I}_N+J\otimes \mathrm{PC})^{-1}
(\mathrm{I}_N\otimes\mathrm{P})).
\end{align}
As all matrices are in lower triangular form, the eigenvalues are the diagonal entries. As all diagonal entries of C are 0, it can be seen that those of $(\mathrm{I}_N+J\otimes \mathrm{PC})$ are all 1. Thus, its inverse exists with diagonal entries equal to one. Finally, we end up with the same condition as in Theorem~\ref{theorem_stabilitycondition_ILC}.
\end{proof}

%%%%%%%%%%%%%%%%%%%%%%%%%%%%%%%%%%%%%%%%%%%%%%%%%%%%%%%%%%%%%%%%%%%%%%%%%%%%%%%%
\section{QUADROTOR EXPERIMENTS}
To verify the effectiveness of the derived multi-agent learning framework, we implemented the proposed algorithm on a group of quadrotors. The vehicle we used is the Parrot AR.Drone~2.0, which comes with a blackbox onboard controller. Its inputs are the desired roll, $\phi_{des}$, and pitch, $\theta_{des}$, Euler angles, the desired turn rate around the vehicle's vertical axis, $\omega_{\mathsf{z}}$, and the desired velocity, $\dot{\mathsf{z}}_{des}$, in global $\mathsf{z}$-direction. Commands are sent at a frequency of~${f=66.67}$Hz. Position and attitude information is provided by a central overhead motion capture camera system. As the camera system and an appropriate state estimator provide all necessary position, velocity and rotation information, an exact input-output linearization can be applied \cite{Zhou2010} canceling out all the nonlinearities. Consequently, the resulting quadrotor dynamics can be approximated by continuous-time double integrators, decoupled for $\mathsf{x}$- and $\mathsf{y}$-direction.
Discretization using the Taylor series expansion with time constant $\Delta t=\frac{1}{f}=0.015 \mathrm{s}$ leads to 
\begin{align}
x_{i,k}(t+1) &= \begin{bmatrix}
1 & 0.015\\
0 & 1
\end{bmatrix}x_{i,k}(t)+\begin{bmatrix}
\frac{1}{2}0.015^2 \\ 0.015
\end{bmatrix}u_{i,k}(t-\tau)\nonumber \\
y_{i,k}(t) &=\begin{bmatrix}
1&0
\end{bmatrix}x_{i,k}(t) \label{eq:dronedynamics_statespace},
\end{align} 
where input $u$ and output $y$ denote acceleration and position in $\mathsf{x}$- or $\mathsf{y}$-direction, respectively, while $\tau$ represents the time delay of the system consisting of delays in input and output signal processing and plant-inherent delays due to the simplifed modeling. As the mathematical model is linear, all delays can
be combined into one term.
Based on \eqref{eq:dronedynamics_statespace}, we choose an underlying consensus feedback controller
\begin{equation}
u_{i}^{\mathrm{FB}}(t)=\frac{2\eta}{\tau_c}\dot e_{i,k}(t) + \frac{1}{\tau_c^2}e_{i,k}(t),
\label{eq:FB_experiment}
\end{equation}
with error function $e_{i,k}(t)$ as defined in \eqref{e(t)}. With positive controller gains, $\eta$ and $\tau_c$, this controller guarantees asymptotic stability for double-integrator agents under the condition that the communication graph contains a spanning tree. As this paper does not focus on the stability of consensus feedback controllers, we refer to \cite{Ren2008a} for further explanations. Approximating the velocity by the difference quotient of the position, the controller can be discretized and written in the form defined in \eqref{eq:u_FB}.

For the iterative learning, a PD-type (proportional and derivative actions) input update rule is used, 
\begin{align}
u_{i,k+1}^{\mathrm{ILC}}(t)=&u_{i,k}^{\mathrm{ILC}}(t)+k_p e_{i,k}(t+r-1)\nonumber\\
&+k_d \frac{e_{i,k}(t+r)-e_{i,k}(t+r-2)}{2\Delta t},
\end{align}
with learning gains, $k_p$ and $k_d$, step size $\Delta t$, and relative degree $r$. The central difference quotient is used for better noise suppression \cite{Chen2002}. This is a special case of \eqref{inputupdate} with 
\begin{align}
L(\timeshift)={\frac{k_d}{2\Delta t}}+k_p\timeshift^{-1} -\frac{k_d}{2\Delta t}\timeshift^{-2}.
\end{align}
To determine the relative degree of the real vehicles, several effects must be taken into account, including underlying dynamics from the onboard controller and from the motors, that were neglected in the modeling, and system time delays mainly due to the wireless communication between the computer and the vehicle.
Since these effects are difficult to measure, we identified the relative degree experimentally. 
The communication delays can also destabilize the closed-loop system with the feedback controller \eqref{eq:FB_experiment}, if the graph contains cycles, see \cite{Munz2008} or \cite{Hu2010}.

We consider a team of two quadrotors with agent $v_1$ getting information from the virtual leader and agent $v_2$ only from agent $v_1$. Due to space and wireless communication limitations, it was not possible to include more agents in the current experimental setup. However, simulations verified that the presented theoretic results work as expected even for larger teams and more complex graphs.
We chose the following setup: controller parameters $\eta = 0.707$ and~${\tau_c = 1.7}$, ILC learning gains $k_p = 0.35$ and $k_d = 17.3$, ILC time shift $r = 49$.
Assuming the time shift matches the relative degree and with the eigenvalues of the corresponding graph Laplacian $\lambda_{1,2}=1$, we can see that \eqref{eq:stabilitycondition_ILC_scalar} holds and thus asymptotic stability is guaranteed.

Figures \ref{fig:timeplots_FB}-\ref{fig:errorConvergence_FB} show the experimental results for the ILC with the underlying consensus feedback controller over 20~iterations. Both quadrotors were flown at the same time in a given formation with a fixed distance apart. For the plots, the distance offset was subtracted. We repeated the whole learning experiment ten times and show the average error convergence and standard deviation in Figure~\ref{fig:errorConvergence_FB}.

Figure~\ref{fig:timeplots_FB} shows the position trajectories over time. It can be seen that in the first iteration, where the ILC input is zero, the first vehicle (in blue) is delayed and shows lower amplitudes. As the second vehicle (in red) only follows the first one, its performance is even worse. After some iterations (see solid lines), both drones learn to track the reference. 
\begin{figure} 
	\centering 
	\includegraphics[width=0.5\textwidth]{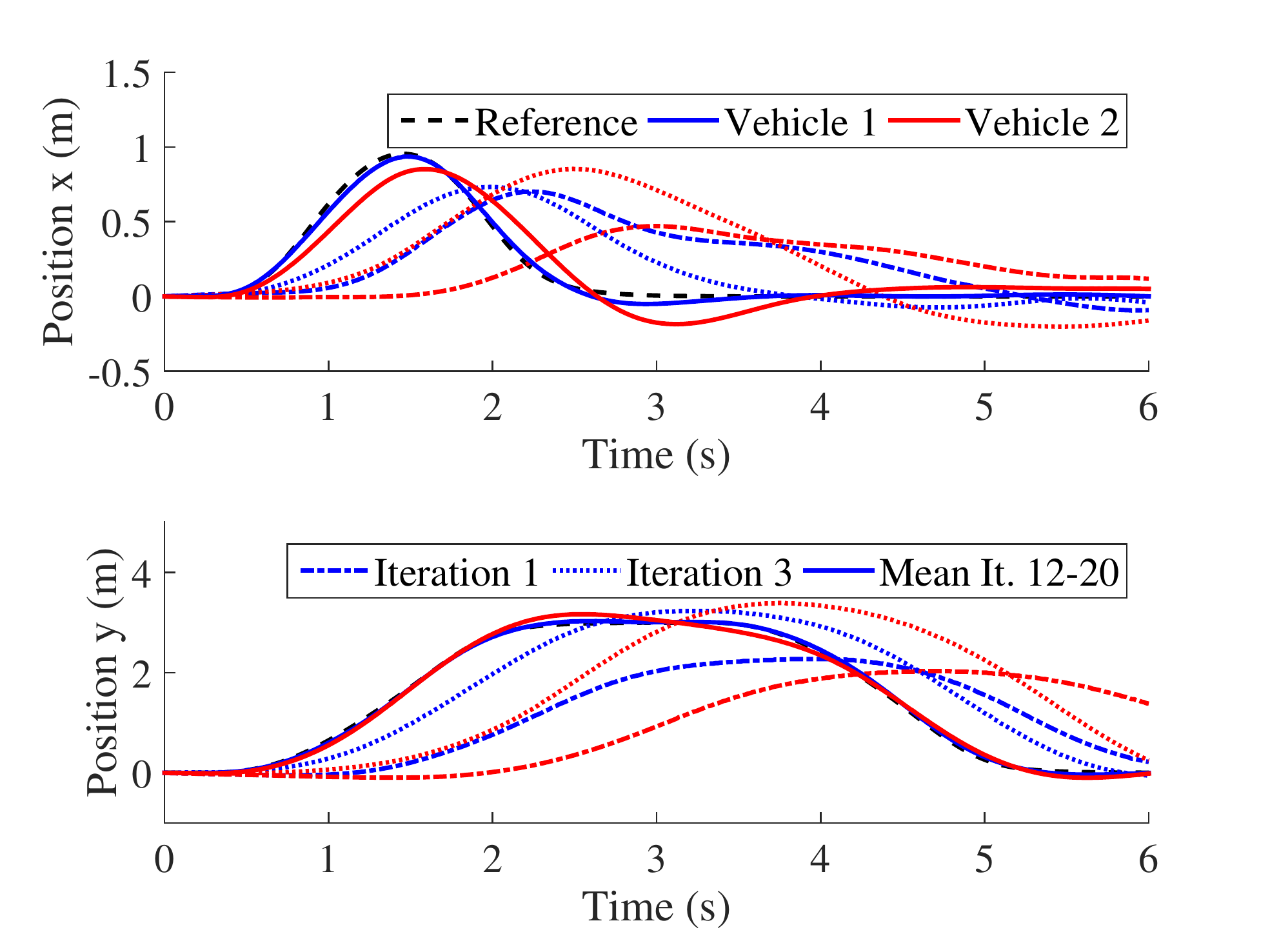}
	\caption{Trajectories over time in $\mathsf{x}$- and $\mathsf{y}$-direction for the ILC algorithm with underlying consensus feedback. Vehicle 1 (blue) and vehicle 2 (red) learn to follow the desired reference trajectory (dashed black). Highlighted are the first (dash-dotted) and the third (dotted) iteration, and the mean over iterations 12-20 (solid).}
	\label{fig:timeplots_FB}
	\vspace{-2mm}
\end{figure}
In Figure~\ref{fig:statespace_FB}, the corresponding workspace trajectories are depicted. Note that the goal was not to track this D-shaped trajectory but to follow the timed reference signal in Figure~\ref{fig:timeplots_FB} separately for $\mathsf{x}$ and $\mathsf{y}$. The performance of both vehicles improves significantly over iterations. However, it can be seen that agent $v_2$ learns slower as it has no access to the desired trajectory.
\begin{figure} 
	\centering 
	\includegraphics[width=0.5\textwidth]{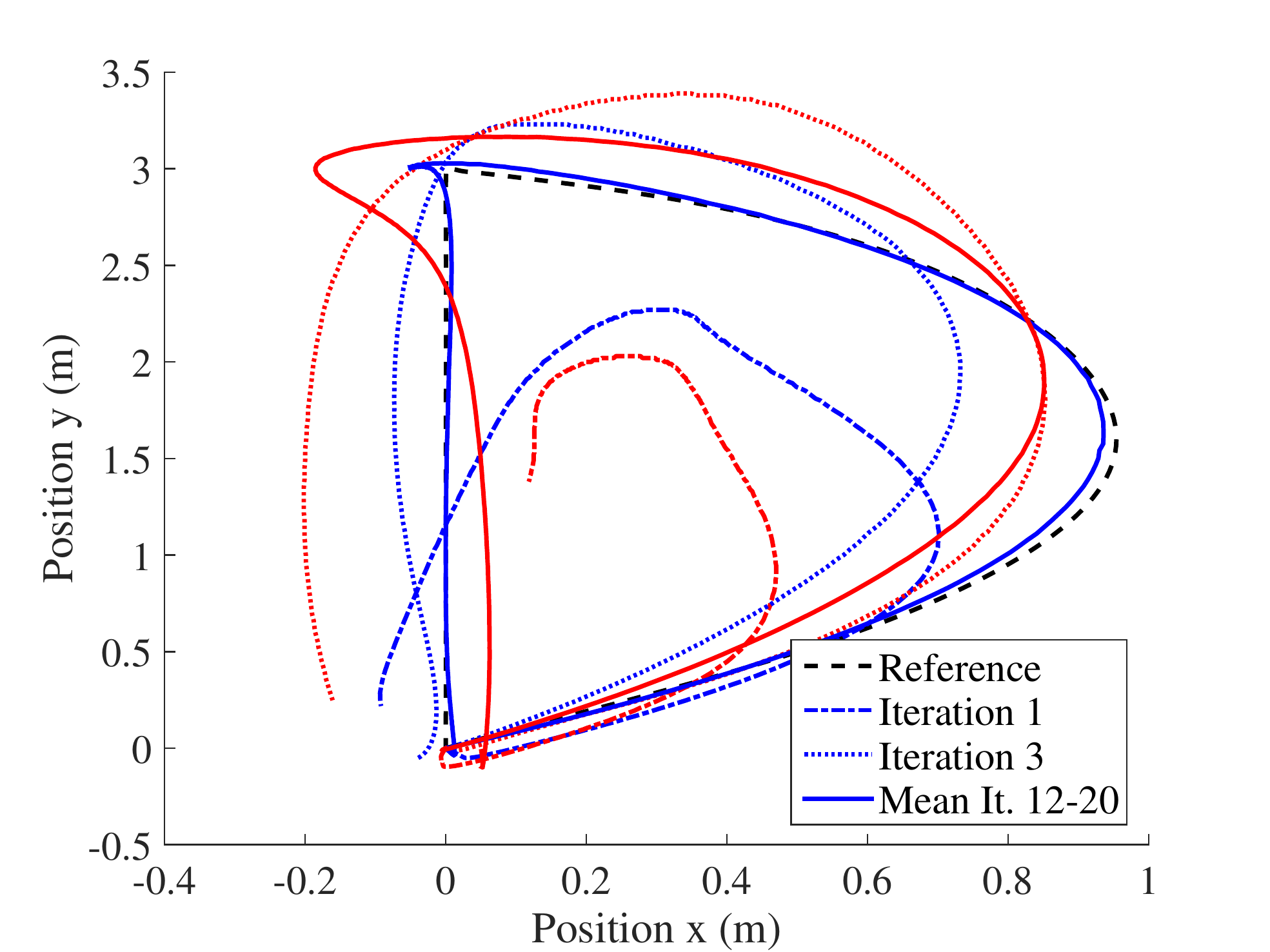}
	\caption{Trajectories in the workspace ($\mathsf{y}$ over $\mathsf{x}$) for the ILC algorithm with underlying consensus feedback. Vehicle 1 (blue) and vehicle 2 (red) significantly improve their performance with respect to the desired reference trajectory (dashed black). Highlighted are the first (dash-dotted) and the third (dotted) iteration, and the mean over iterations 12-20 (solid).}
	\label{fig:statespace_FB}
	\vspace{-2mm}
\end{figure}
Figure \ref{fig:inputs} shows the corresponding input trajectories. For space reasons only the $\mathsf{x}$-direction is plotted. It can be seen that, initially, the ILC input is zero and the consensus feedback component dominates. Whereas after convergence is reached, the feedback input is nearly zero and mainly compensates for non-repetitive errors, while the ILC feedforward input compensates for the large repetitive error. Comparing the converged ILC input with the initial, purely feedback-based input shows that ILC causes larger input magnitudes with peaks being time-shifted to the left. Instead of being reactive, the ILC is proactive and sends aggressive input signals that keep the vehicle on track.

\begin{figure} 	
	\centering 
	\includegraphics[width=0.5\textwidth]{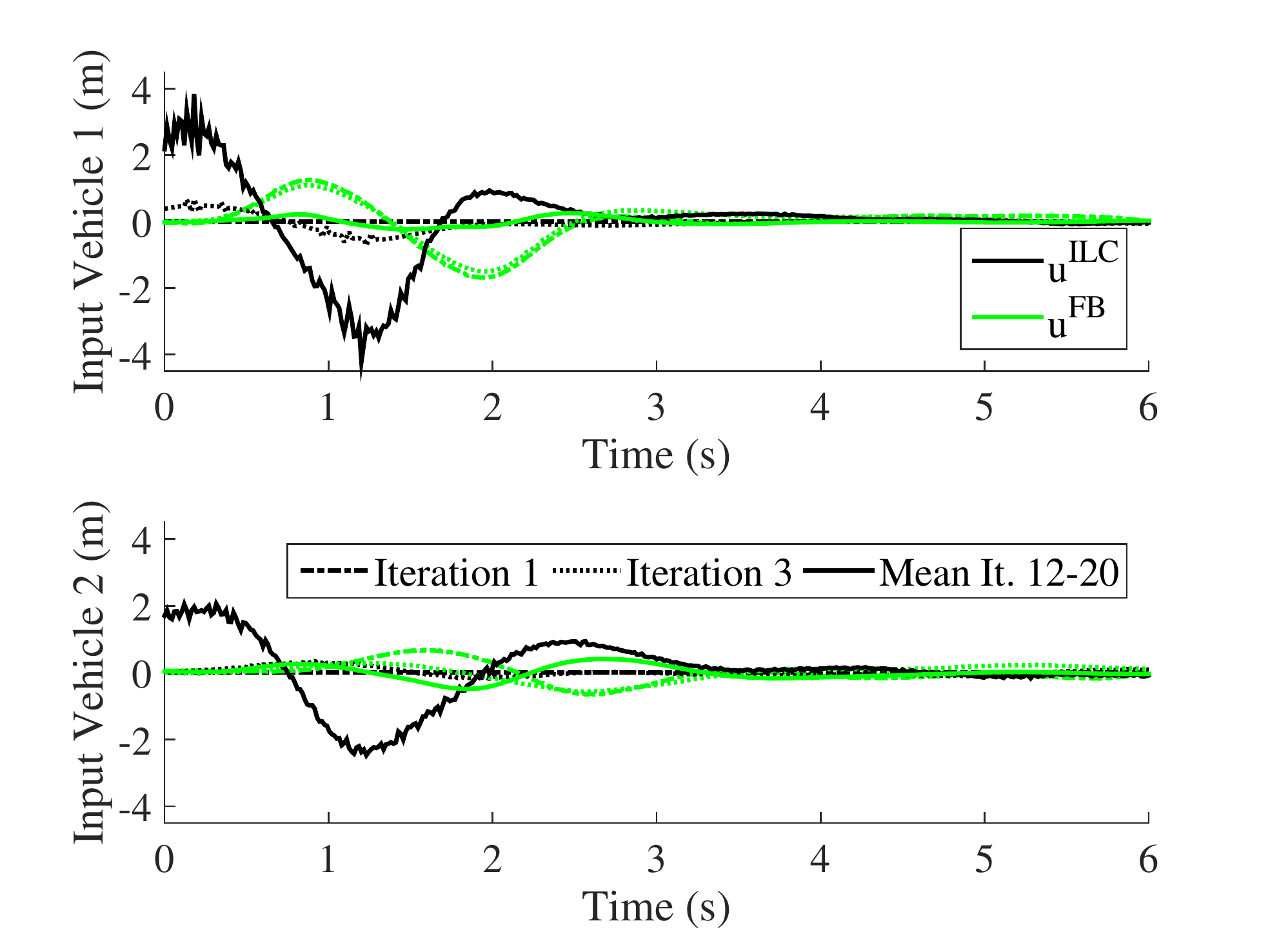}
	\caption{Input trajectories in $\mathsf{x}$-direction for vehicle 1 (top) and vehicle~2 (bottom). The ILC-generated feedforward input (black) increases over iterations to compensate for the learned repetitive disturbances, while the input from the consensus feedback controller (green) decreases and only accounts for non-repetitive disturbances at the end. Highlighted are the first (dash-dotted) and the third (dotted) iteration and the mean over iterations 12-20 (solid).}
	\label{fig:inputs}
	\vspace{-2mm}
\end{figure}
The learning performance can be deduced from the convergence of the errors \eqref{e(t)} over iterations shown in Figure~\ref{fig:errorConvergence_FB}. The error of agent $v_2$ (bottom) is computed relative to agent~$v_1$~(top); that is, it describes the formation error. Let us first consider the case with consensus feedback enabled~(magenta). It can be seen that vehicle $v_1$ learns faster than vehicle $v_2$ due to the direct access to the reference, increasing the relative error in iteration~2. Accordingly, not only disturbances affecting the second agent but also the first one lead to increasing formation errors. This also explains the slightly higher error of agent $v_2$ after convergence is reached. 

\begin{figure} 
	\centering 
	\includegraphics[width=0.5\textwidth]{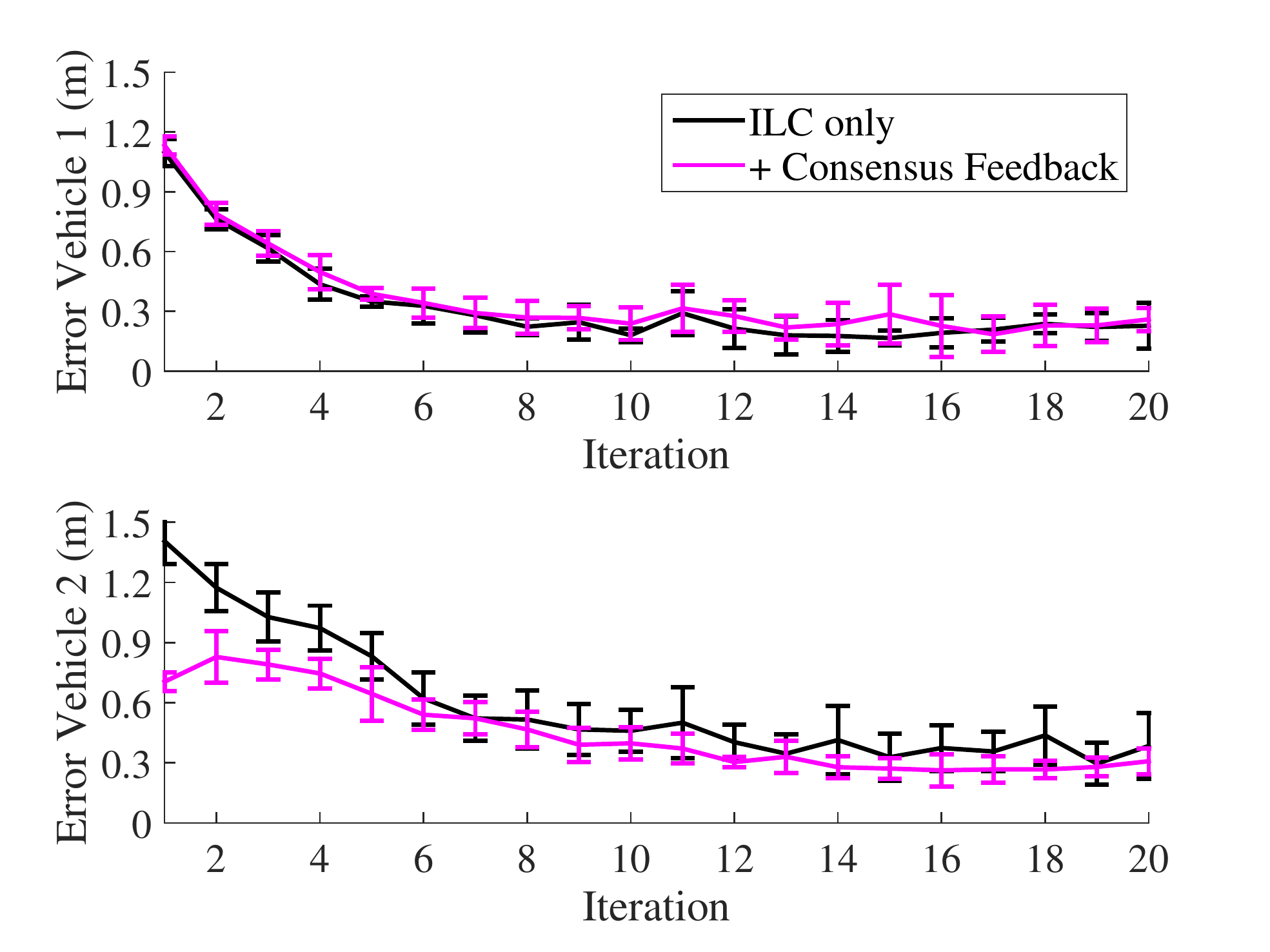}
	\caption{Error convergence plots over iterations for the ILC algorithm with (magenta) and without (black) the additional consensus feedback. The error is computed as $\frac{1}{T}\sum_{T}\left(|e_{i,k}^{\mathsf{x}}(t)|+|e_{i,k}^{\mathsf{y}}(t)|\right)$ with $e_{i,k}(t)$ as in \eqref{e(t)} for $\mathsf{x}$- and $\mathsf{y}$-direction respectively. The solid lines are the mean errors for agent $v_1$ (top) and agent $v_2$ (bottom) over ten experiments each. The bars denote the standard deviations.}
	\label{fig:errorConvergence_FB}
	\vspace{-2mm}
\end{figure}
For comparison, we did the same learning experiments with the consensus feedback controller \eqref{eq:FB_experiment} disabled (black lines). That is, each vehicle's position feedback controller is based only on the vehicle's own tracking error. It can be seen that the values for agent $v_1$ are almost the same with and without the consensus feedback controller. The error of agent $v_2$, which can be interpreted as the formation error, and its standard deviation decrease significantly if the consensus controller is enabled. Since the second vehicle has no reference information and does not follow the leader without the consensus feedback, it does not move at all in the first iteration, which leads to the high initial error. With the consensus feedback controller enabled, the relative error of agent $v_2$ after convergence (that is, after iteration 15) decreases by 24\% and its standard deviation by 53\%. As a result, the proposed distributed feedback has a positive impact on the performance of formation flying as it does both \textit{(i)} reducing the formation error in the first iterations, which can help to avoid collisions, and \textit{(ii)} accounting for non-repetitive, relative disturbances during iterations, which reduces the tracking error and variance after learning convergence.

To show the improvements over the D-type learning function, we compared the experimental results without the feedback controller with a D-type ILC where $k_p = 0$. Without the P-gain, the learning convergence was slower for both agents and, especially for $v_2$, the converged error was much higher (+60\%). For space reasons, plots are not shown here.

%\begin{figure} 
%	\centering 
%	\includegraphics[width=0.5\textwidth]{mean_bars.png}
%	\caption{Comparison of position errors for the ILC experiments with and without the distributed, consensus feedback controller \eqref{eq:FB_experiment}. The blue bars denote agent $v_1$, the red ones agent $v_2$. The light bars are the mean values of the error after convergence (iterations 12-20) over ten experiments, the dark bars the corresponding standard deviations, see Figure \ref{fig:errorConvergence_FB}. The relative error between vehicle 1 and 2 (see red bars) significantly improves when adding the consensus controller.}
%	\label{fig:error_meanstdbars}
%\end{figure}

%%%%%%%%%%%%%%%%%%%%%%%%%%%%%%%%%%%%%%%%%%%%%%%%%%%%%%%%%%%%%%%%%%%%%%%%%%%%%%%%
\section{CONCLUSIONS}
We developed a distributed ILC algorithm for multi-agent systems, which allows for arbitrary linear and causal learning functions. As a result, we were able to consider a PD-type input update rule extending previous work found in literature that was restricted to learning functions depending only on the error derivative (D-type). Since it can compensate for position offsets, the proposed approach leads to better tracking performance and lower errors. Furthermore, many other learning function options are possible. We derived a simple scalar condition for stability of the proposed learning algorithm in theory. However, to achieve a good learning performance in practice, parameter tuning in simulations and experiment was necessary.

As ILC only compensates for repetitive disturbances, we included a consensus-based feedback controller to increase robustness against non-repetitive disturbances and noise. That this feedback controller does not affect stability of the ILC algorithm was proven theoretically. Moreover, it was shown that the same holds for any dynamic coupling between neighboring agents.
Experimental results showed that the resulting distributed feedback and learning algorithm achieves better reference tracking and lower formation
error, compared to the case without the consensus feedback. As the feedback controller decreases the influence of non-repetitive disturbances, better overall formation tracking performance is achieved.

%%%%%%%%%%%%%%%%%%%%%%%%%%%%%%%%%%%%%%%%%%%%%%%%%%%%%%%%%%%%%%%%%%%%%%%%%%%%%%%%

\bibliographystyle{ieeetr}
\bibliography{MA_AndiHock}

\begin{thebibliography}{10}

\bibitem{WeiRen2005}
W.~Ren, R.~W. Beard, and E.~M. Atkins, ``{A survey of consensus problems in
  multi-agent coordination},'' in {\em Proc. of the American Control Conference
  (ACC)}, pp.~1859--1864, 2005.

\bibitem{Arimoto1984}
S.~Arimoto, S.~Kawamura, and F.~Miyazaki, ``{Bettering operation of robots by
  learning},'' {\em Journal of Robotic Systems}, vol.~1, no.~2, pp.~123--140,
  1984.

\bibitem{Bristow2006b}
D.~A. Bristow, M.~Tharayil, and A.~G. Alleyne, ``{A survey of iterative
  learning control},'' {\em IEEE Control Systems Magazine}, vol.~26, no.~3,
  pp.~96--114, 2006.

\bibitem{Schoellig2012a}
A.~P. Schoellig, F.~L. Mueller, and R.~D'Andrea, ``{Optimization-based
  iterative learning for precise quadrocopter trajectory tracking},'' {\em
  Autonomous Robots}, vol.~33, no.~1-2, pp.~103--127, 2012.

\bibitem{Ahn2009a}
H.-S. Ahn and Y.~Chen, ``{Iterative learning control for multi-agent
  formation},'' in {\em Proc. of ICROS-SICE International Joint Conference},
  pp.~3111--3116, 2009.

\bibitem{Yang2012c}
S.~Yang, J.-X. Xu, and D.~Huang, ``{Iterative learning control for multi-agent
  systems consensus tracking},'' in {\em Proc. of the IEEE Conference on
  Decision and Control (CDC)}, pp.~4672--4677, 2012.

\bibitem{Meng2012d}
D.~Meng, Y.~Jia, J.~Du, and F.~Yu, ``{Tracking control over a finite interval
  for multi-agent systems with a time-varying reference trajectory},'' {\em
  Systems {\&} Control Letters}, vol.~61, no.~7, pp.~807--818, 2012.

\bibitem{Liu2012a}
Y.~Liu and Y.~Jia, ``{An iterative learning approach to formation control of
  multi-agent systems},'' {\em Systems {\&} Control Letters}, vol.~61, no.~1,
  pp.~148--154, 2012.

\bibitem{Li2014b}
J.~Li and J.~Li, ``{Adaptive iterative learning control for coordination of
  second-order multi-agent systems},'' {\em International Journal of Robust and
  Nonlinear Control}, vol.~24, pp.~3282--3299, 2014.

\bibitem{Meng2015a}
D.~Meng, Y.~Jia, and J.~Du, ``{Robust iterative learning protocols for
  finite-time consensus of multi-agent systems with interval uncertain
  topologies},'' {\em International Journal of Systems Science}, vol.~46,
  no.~5, pp.~857--871, 2015.

\bibitem{Yang2014a}
S.~Yang, J.-X. Xu, D.~Huang, and Y.~Tan, ``{Optimal iterative learning control
  design for multi-agent systems consensus tracking},'' {\em Systems {\&}
  Control Letters}, vol.~69, pp.~80--89, 2014.

\bibitem{Mesbahi2010}
M.~Mesbahi and M.~Egerstedt, {\em {Graph theoretic methods in multiagent
  networks}}.
\newblock Princeton: Princeton University Press, 2010.

\bibitem{Norrlof2002}
M.~Norrl{\"{o}}f and S.~Gunnarsson, ``{Time and frequency domain convergence
  properties in iterative learning control},'' {\em International Journal of
  Control}, vol.~75, no.~14, pp.~1114--1126, 2002.

\bibitem{Zhou2010}
Q.~L. Zhou, Y.~Zhang, C.~A. Rabbath, and D.~Theilliol, ``{Design of feedback
  linearization control and reconfigurable control allocation with application
  to a quadrotor UAV},'' in {\em Proc. of the IEEE Conference on Control and
  Fault Tolerant Systems}, pp.~371--376, 2010.

\bibitem{Ren2008a}
W.~Ren and R.~W. Beard, {\em {Distributed consensus in multi-vehicle
  cooperative control}}.
\newblock Springer, 2008.

\bibitem{Chen2002}
Y.~Chen and K.~L. Moore, ``{An optimal design of PD-type iterative learning
  control with monotonic convergence},'' in {\em Proc. of the IEEE
  International Symposium on Intelligent Control}, pp.~55--60, 2002.

\bibitem{Munz2008}
U.~M{\"{u}}nz, A.~Papachristodoulou, and F.~Allg{\"{o}}wer, ``{Delay-dependent
  rendezvous and flocking of large scale multi-agent systems with communication
  delays},'' in {\em Proc. of the IEEE Conference on Decision and Control
  (CDC)}, pp.~2038--2043, 2008.

\bibitem{Hu2010}
J.~Hu and Y.~Lin, ``{Consensus control for multi-agent systems with
  double-integrator dynamics and time delays},'' {\em IET Control Theory {\&}
  Applications}, vol.~4, no.~1, pp.~109--118, 2010.

\end{thebibliography}

\end{document}